\pdfoutput=1
\documentclass[letterpaper]{article} 
\usepackage{aaai24}  
\usepackage{times}  
\usepackage{helvet}  
\usepackage{courier}  
\usepackage[hyphens]{url}  
\usepackage{graphicx} 
\usepackage{natbib}  
\usepackage{caption} 
\usepackage{algorithm}
\usepackage{algorithmic}
\usepackage[utf8]{inputenc}
\usepackage[T1]{fontenc}
\usepackage{newfloat}
\usepackage{listings}
\DeclareCaptionStyle{ruled}{labelfont=normalfont,labelsep=colon,strut=off} 
\floatstyle{ruled}
\newfloat{listing}{tb}{lst}{}
\floatname{listing}{Listing}
\usepackage{latexsym, amsmath, amssymb, amsthm, amsfonts, mathtools}
\usepackage{enumerate}
\usepackage[nameinlink, noabbrev, capitalise]{cleveref}
\allowdisplaybreaks

\theoremstyle{plain}
\newtheorem{theorem}{Theorem}

\newtheorem{proposition}{Proposition}

\theoremstyle{definition}
\newtheorem{definition}{Definition}
\newtheorem{observation}{Observation}

\newtheorem{remark}{Remark}

\DeclareMathOperator*{\argmax}{arg\,max}

\newcommand{\ind}{\mathbf{1}}

\newcommand{\E}{\mathbb{E}}

\newcommand{\brac}[1]{\left[ #1 \right]}
\newcommand{\set}[1]{\left\{ #1 \right\}}

\newcommand{\inner}[1]{\left\langle #1 \right\rangle}

\newcommand{\R}{\mathcal{R}}
\renewcommand{\S}{\mathcal{S}}
\newcommand{\A}{\mathcal{A}}
\renewcommand{\O}{\mathcal{O}}

\newcommand{\mS}{\overline{\mathcal{S}}}
\newcommand{\mA}{\overline{\mathcal{A}}}
\newcommand{\mP}{\overline{P}}
\newcommand{\mr}{\overline{r}}

\newcommand{\mM}{\overline{M}}
\newcommand{\mG}{\overline{G}}

\newcommand{\mmu}{\overline{\mu}}
\newcommand{\ms}{\overline{s}}
\newcommand{\ma}{\overline{a}}
\newcommand{\hist}{\mathcal{H}}

\newcommand{\sdag}{s^{\dagger}}
\newcommand{\adag}{a^{\dagger}}
\newcommand{\rdag}{r^{\dagger}}
\newcommand{\odag}{o^{\dagger}}

\title{Optimal Attack and Defense for Reinforcement Learning}
\author{
    Jeremy McMahan,
    Young Wu,
    Xiaojin Zhu,
    Qiaomin Xie
}
\affiliations{
    University of Wisconsin-Madison\\

    jmcmahan@wisc.edu, yw@cs.wisc.edu, jerryzhu@cs.wisc.edu, qiaomin.xie@wisc.edu
}

\begin{document}

\maketitle

\begin{abstract}
To ensure the usefulness of Reinforcement Learning (RL) in real systems, it is crucial to ensure they are robust to noise and adversarial attacks. In adversarial RL, an external attacker has the power to manipulate the victim agent's interaction with the environment. We study the full class of online manipulation attacks, which include (i) state attacks, (ii) observation attacks (which are a generalization of perceived-state attacks), (iii) action attacks, and (iv) reward attacks. We show the attacker's problem of designing a stealthy attack that maximizes its own expected reward, which often corresponds to minimizing the victim's value, is captured by a Markov Decision Process (MDP) that we call a meta-MDP since it is not the true environment but a higher level environment induced by the attacked interaction. We show that the attacker can derive optimal attacks by planning in polynomial time or learning with polynomial sample complexity using standard RL techniques. We argue that the optimal defense policy for the victim can be computed as the solution to a stochastic Stackelberg game, which can be further simplified into a partially-observable turn-based stochastic game (POTBSG). Neither the attacker nor the victim would benefit from deviating from their respective optimal policies, thus such solutions are truly robust. Although the defense problem is NP-hard, we show that optimal Markovian defenses can be computed (learned) in polynomial time (sample complexity) in many scenarios.

\end{abstract}

\section{Introduction}\label{sec: intro}

Reinforcement Learning (RL) has become a staple with a plethora of applications including the breakthrough ChatGPT~\cite{InstructGPT}. With the growth of RL applications, it is critical to understand the security threats posed to RL and how to defend against them. In many applications, noisy measurements can cause the agent-environment interaction to evolve entirely differently than what one would expect in theory. Even worse, malicious attackers can strategically modify the agent-environment interaction to induce catastrophic outcomes for the agent. If RL methods are to be used in diverse and critical settings, it is essential to ensure these RL algorithms are robust to potential attacks.  

In adversarial RL, a victim agent interacts with an environment while being disrupted by an attacker. The attacker has the power to manipulate each aspect of the victim-environment interaction. In particular, the attacker can change: (i) the environment's state (\emph{state attacks}), (ii) the victim's observation (\emph{observation attacks}), (iii) the action taken by the victim (\emph{action attacks}), and (iv) the reward received by the victim (\emph{reward attacks}). When the environment is fully-observable, observation attacks translate to well-studied \emph{perceived-state attacks}. We refer to all of these attack surfaces by \emph{online manipulation attacks}. The attacker may use a subset or all of these attack surfaces to optimize its own expected reward from the attack, which often corresponds to minimizing the victim's value. However, the attacker cannot perform arbitrary manipulations without raising suspicion. Hence, the attacker must restrict its manipulations to a predefined set of stealthy attacks. On the other hand, the attacker-aware victim seeks to choose a \emph{defense} policy whose value is provably robust even under the worst possible stealthy attack.

From the attacker's perspective, it faces an optimal control problem: it needs to strategically choose stealthy attacks to optimize its value. Unlike typical control problems, the attacker must deal with the uncertainty of the victim's actions in addition to that of the stochastic environment. Thus, the attacker's problem involves a multi-agent feature. For any fixed victim policy $\pi$, we can view the attacker's problem as computing its best response attack to the victim's chosen $\pi$. 
From the victim's perspective, we argue it faces a Stochastic Stackelberg game: it needs to choose a policy that achieves maximum value in the environment under the attacker's best-response attack. A defense policy designed following this principle ensures neither the victim nor the attacker would benefit from deviating from their chosen policies, and so an equilibrium would be achieved. This implies that regardless of the attack, the defense policy always achieves at least the game's optimal value. However, computing optimal Stackelberg strategies for stochastic games is NP-hard. Thus, both the attacker and the victim are faced with challenging optimization problems.

Although the attack and defense problems are of great importance, complete solutions have yet to be discovered. For the attack problem, most works focus on the empirical aspects, lacking theoretical guarantees. Provably optimal attacks have only been devised for the special case of test-time, perceived-state attacks~\cite{Proutiere2021,Zhang2020,Sun2022}. The situation is even worse for the defense problem, which is arguably more important. Nearly all proposed defenses are designed to be effective against a specific, known attack. This results in a cat-and-mouse game: the attacker can just design a new attack for the given defense policy and so the victim would always be at risk. In addition, the two approaches with provable guarantees are restricted to the planning, reward-poisoning setting~\cite{banihashem2021defense}, and the test-time, perceived-state attack setting~\cite{Zhang2021}. Furthermore, neither defense can be computed efficiently and it is unrealistic to assume the victim knows the attacker's exact algorithm.

\paragraph{Our Contributions.} Despite the challenges of the attack and defense problems, we develop frameworks for computing optimal attacks and defenses for any combination of attack surfaces, which are provably efficient in many cases. From the attacker's side, we show that for any fixed victim policy, the optimal attack can be computed as the solution to another Markov Decision Process (MDP). We call this environment a \emph{meta-MDP} since it is not the true environment, but is a higher-level environment induced by the victim-attacker-environment interaction. Importantly, the attacker can simulate an interaction with the meta-MDP by interacting with the victim and the true environment. Hence, the attacker can attack optimally by solving the meta-MDP using any standard MDP planning or RL algorithms. In addition, we show that the size of the meta-MDP is polynomial in the size of the original environment and the size of the victim's policy. Thus, optimal attacks can always be computed or learned efficiently. Our framework also extends to linear MDPs. Hence, we provide the first provably optimal attacks for beyond perceived-state attacks and  the first provably optimal attacks for the linear setting, all of which can be computed in polynomial time. We note our framework also solves the certifying robustness problem posed in ~\cite{crop}.

On the victim's side, we argue that the defense problem is most naturally modeled by a stochastic Stackelberg game~\cite{SSG}, which can be captured by a much simpler partially-observable turn-based stochastic game (POTBSG)~\cite{TBSGtheory}. Thus, the victim can compute its optimal robust defense by finding a weak Stackelberg equilibrium (WSE) for the meta-POTBSG. Again, the victim can simulate the meta-POTBSG by interacting with the attacker and the true environment. When the attacker is adversarial, the victim can defend optimally by solving the meta-POTBSG using any standard zero-sum POTBSG planning or distributed learning algorithms. Unlike the attack problem, we show that the victim's defense problem is NP-hard in general even to find approximate solutions when observation attacks are permitted. However, we show that optimal Markovian defenses can be computed efficiently when excluding observation attacks by exploiting the sequential nature of the attacks. This gives a broad class of games for which WSE is computable. Overall, we present the first-ever provable defense algorithms for both the planning and learning settings and show our defenses can be computed efficiently for a broad class of instances.

\subsection{Related Work}
Many prior works have studied adversarial RL under various models and objectives. Amongst the first works,~\citet{Behzadan2017, Huang2017,Kos2017} study perceived-state attacks through the lens of adversarial examples for deep neural nets ~\cite{advexamples}.~\citet{Kos2017} also considers adversarial examples, but with the goal of minimizing the number of attacks needed  to achieve large damage. These works focused on achieving large damage at the current time. Later ~\citet{Lin2017, Sun2020} developed more advanced heuristics that incorporate future value into their attacks to achieve long-term damage. Meanwhile,~\citet{Tretschk2018} trained an adversarial deep net to compute perturbations that allows other objectives for the attacker.

Afterward, many works began considering the objective of maximizing the damage to the victim rather than minimizing the number of attacks.~\citet{Proutiere2021,Zhang2020} developed optimal algorithms for computing perceived-state attacks. Both works formulated the attack problem as a different MDP as we do here.~\citet{Sun2022} formulated an actor-director model for the attack problem that is easier to solve for some MDPs and retains guarantees of optimality. The idea of adversarial training was then used in conjunction with the attack formulation from~\cite{Zhang2020} to obtain experimentally robust victim policies~\cite{Zhang2021}.

Action and reward attacks have been considered heavily in the training-time setting. For example, ~\citet{actionatt2train,actionattspecial} considered action attacks. Reward poisoning attacks are the focus of the work by ~\citet{rewardpoison,trainactionreward}. In fact, a combination action and reward attack are devised by ~\citet{trainactionreward}. Most of these works consider the policy teaching setting, where the attacker's goal is for the victim to follow a fixed policy $\pi^{\dagger}$. Some algorithms achieve sublinear regret for the attacker when the victim policy is no regret ~\cite{minregret-action}; though, none compute truly optimal attacks.

\section{Attack Surfaces}\label{sec: surfaces}

\paragraph{POMDPs.} We denote a infinite-horizon discounted environment POMDP by $M = (\S, \O, \A, P, R, \allowbreak \gamma, \mu)$ where (i) $\S$ is the state set, (ii) $\O$ is the observation set, (iii) $\A$ is the action set, (iv) $P : \S \times \A \to \Delta(\S)$ is the transition kernel, (v) $R : \S \times \A \to \Delta(\mathbb{R})$ is the reward distribution, (vi) $\gamma$ is the discount factor, and (vii) $\mu \in \Delta(S)$ is the initial state distribution. We let $\O(s)$ denote the distribution of observations at state $s$. We also let $\R$ denote the set of all supported rewards. The total expected reward the victim receives from following policy $\pi$ in environment $M$ is its \emph{value}, i.e., the expected cumulative discounted rewards $V^{\pi}_M : = \E_{M}^{\pi}\brac{\sum_{t = 0}^{\infty} \gamma^t r(s_t,a_t)}$. 

Suppose the victim interacts with a Markovian environment, $M$, using a fixed stationary, Markovian policy $\pi : \O \to \Delta(A)$. At any time $t$, let $s_t$ denote $M$'s current state and $o_t$ denote the generated observation. In the standard setting, the victim chooses an action $a_t \sim \pi(o_t)$ and then receives a reward $r_t \sim R(s_t,a_t)$. Afterwards, $M$ transitions to its next state $s_{t+1} \sim P(s_t,a_t)$. We see there are several points during time $t$ at which information is exchanged between the victim and $M$. We further break down the interaction at time $t$ based on these points of information exchange, which we call \emph{subtimes}:
\begin{enumerate}
    \item At the first subtime, $t_1$, $M$ receives its state $s_{t} \sim P(s_{t-1},a_{t-1})$.
    \item At the second subtime, $t_2$, the victim receives its observation $o_t \sim \O(s_t)$.
    \item At the third subtime, $t_3$, $M$ receives the victim's action $a_t \sim \pi(o_t)$.
    \item At the fourth subtime, $t_4$, the victim receives its reward $r_t \sim R(s_t,a_t)$.
\end{enumerate}

\paragraph{Online Attacks.} In the adversarial setting, a third-party called the \emph{attacker} interferes with the victim-$M$ interaction. Here, the attacker may intercept and then corrupt the information being exchanged between the victim and environment $M$. The attacker has access to four attack surfaces:
\begin{enumerate}
    \item (\emph{State Attack}) A state attack changes the state of $M$ from $s_t$ to $\sdag_t$. The attack influences the observation $o_t \sim \O(\sdag_t)$. If $M$ receives action $a_t$, the attack also influences the reward $r_t \sim R(\sdag_t, a_t)$ and the next state $s_{t+1} \sim P(\sdag_t, a_t)$.
    \item (\emph{Observation Attack}) An observation attack causes the victim to receive observation $\odag_t$ instead of $o_t \sim \O(s_t)$. The attack influences the victim's action $a_t \sim \pi(\odag_t)$.
    \item (\emph{Action Attack}) An action attack causes $M$ to receive action $\adag_t$ instead of $a_t$. The attack influences the reward $r_t \sim R(s_t, \adag_t)$ and next state $s_{t+1} \sim P(s_t, \adag_t)$.
    \item (\emph{Reward Attack}) A reward attack causes the victim to receive reward $\rdag_t$ instead of reward $r_t \sim R(s_t, a_t)$. 
\end{enumerate}
We call each of these attack surfaces as \emph{online manipulation attacks}. These attack surfaces in conjunction give the attacker the power to corrupt every element of the triple $(s,a,r)$ that define the interaction between the victim and $M$. 

If $M$ is fully observable, observation attacks correspond to \emph{perceived-state} attacks, which change what the victim thinks is $M$'s state. Notice unlike the other surfaces, state attacks could be performed at two different subtimes. Namely, the attacker can change the state before $M$ transitions at $t_1$ or before $M$ receives the victim's action at $t_3$. For simplicity, we assume state attacks only happen at $t_1$, but our results apply equally well to both versions.

\paragraph{Adversarial Interaction.} Overall, the victim-attacker-$M$ interaction at time $t$ now evolves as follows:
\begin{enumerate}
    \item At subtime $t_1$, $M$ is in state $s_t$.
    \begin{enumerate}
        \item Attacker: changes $s_t$ to $\sdag_t$.
        \item $M$: enters state $\sdag_{t}$ and generates observation $o_{t} \sim \O(\sdag_{t})$.
    \end{enumerate}
    \item At subtime $t_2$, $M$ is in state $\sdag_t$ and has generated observation $o_t$.
    \begin{enumerate}
        \item Attacker: changes $o_t$ to $\odag_t$.
        \item Victim: chooses action $a_t \sim \pi(\odag_t)$.
    \end{enumerate}
    \item At subtime $t_3$, $M$ is in state $\sdag_t$ and the victim chose action $a_t$.
    \begin{enumerate}
        \item Attacker: changes $a_t$ to $\adag_t$.
        \item $M$: generates reward $r_t \sim R(\sdag_t,\adag_t)$ and generates state $s_{t+1} \sim P(\sdag_t, \adag_t)$.
    \end{enumerate}
    \item At subtime $t_4$, $M$ has generated reward $r_t$.
    \begin{enumerate}
        \item Attacker: changes $r_t$ to $\rdag_t$.
        \item Victim: receives reward $\rdag_t$.
    \end{enumerate}
\end{enumerate}
This process then repeats starting from $s_{t+1}$.

\paragraph{Attacker Constraints.} In general, the attacker may not arbitrarily manipulate the interaction. For example, some attacks may be physically impossible or risk detection. As such, we assume the attacker has a set $\mathcal{B}$ that defines the feasible manipulations it can perform. For example, the attacker might require a manipulated observation to be visually similar to the true observation. Thus, the set of feasible observation attacks should depend on the true observation. Applying the same logic to each attack surface, we see the feasible attack sets should take the form: $\mathcal{B}(s) \subseteq \S$, $\mathcal{B}(o) \subseteq \O$, $\mathcal{B}(a) \subseteq \A$, and $\mathcal{B}(r) \subseteq \R$. However, in some cases, the feasibility of an attack would depend on the interaction before the attack, not just the current element being manipulated. To be fully general, we allow the feasibility sets to take the form: at subtime $t_1$, $\mathcal{B}(s) \subseteq \S$; at subtime $t_2$, $\mathcal{B}(s,o) \subseteq \O$; at subtime $t_3$, $\mathcal{B}(s,o,a) \subseteq \A$; and, at subtime $t_4$, $\mathcal{B}(s,o,a,r) \subseteq \R$.

\section{Optimal Attacks}\label{sec: attack}

\paragraph{Attacker's Goal.} We saw how an attacker can disrupt an interaction but have yet discussed why it would do this. Suppose the attacker has a reward function $g(s,a,r)$ that depends on the victim's received reward, possibly in addition to $M$'s state and the victim's action. The attacker's goal is then to construct an attack that maximizes its own expected reward. Commonly, an attacker just wants to minimize the victim's expected reward under attack, or equivalently maximize the damage to the victim's expected reward. In this case, the attacker's reward function is $g(s,a,r) = -r$. Alternatively, the attacker may want the victim to behave in a specified way. This goal is equivalent to the attacker wanting the victim to choose actions that match a fixed target policy $\pi^{\dagger}$ as often as possible. In this case, the attacker's reward function is $g(s,a,r) = \ind\set{a = \pi^{\dagger}(s)}$. 

\begin{definition}[Attack Problem]\label{def: attack}
    For any $\pi$, the attacker's seeks a policy $\nu^* \in N$ that maximizes its expected reward from the victim-attacker-$M$ interaction:
    \begin{align}
        &\nu^* \in \argmax_{\nu \in N} \E^{\pi,\nu}_{M} \brac{\sum_{t = 0}^{\infty} \gamma^t g(s_t,a_t,r_t)}.
    \end{align}
\end{definition}

We show that the attacker's problem is captured by a MDP. The key insight is that by defining the attacker's state set to capture the results of previous attacks from $t_1$ up to the current subtime, then each attack becomes Markovian with respect to the expanded state set. This is not a significant burden on the attacker since it would need to keep track of this information anyway to compute the feasible attack sets. Thus, the attacker just needs to keep track of the information within a time step to compute optimal attacks.

\begin{definition}[Meta-MDP]\label{def: meta-MDP}
For any victim policy $\pi$, the attacker's meta-MDP is $\mM = (\mS, \mA, \mP, \mr, \overline{\gamma}, \mmu)$ where, 

\begin{itemize}
    \item $\mS = \S \cup (\S \times \O) \cup (\S \times \O \times \A) \cup (\S \times \O \times \A \times \R)$.
    
    \item $\mA(s) = \mathcal{B}(s)$, $\mA(s,o) = \mathcal{B}(s,o)$, $\mA(s,o,a) = \mathcal{B}(s,o,a)$, and $\mA(s,o,a,r) = \mathcal{B}(s,o,a,r)$.

    \item The transitions vary per subtime. Let $\ms \in \mS$, $\ma \in \mA(\ms)$, and $\ms' \in \mS$.

    \begin{enumerate}
        \item If $\ms = s$, then $\ma = \sdag$ and $\ms' = (\sdag, o)$: \(
            \mP(\ms ' \mid \ms, \ma) = \O(o \mid \sdag).
        \)
        \item If $\ms = (s,o)$, then $\ma = \odag$ and $\ms' = (s, \odag, a)$: 
        \(
            \mP(\ms' \mid \ms, \ma) = \pi(a \mid \odag).
        \)
        \item If $\ms = (s,o,a)$, then $\ma = \adag$ and $\ms' = (s, o , \adag, r)$: 
        \(
            \mP(\ms' \mid \ms, \ma) = R(r \mid s, \adag).
        \)
        \item If $\ms = (s,o,a,r)$, then $\ma = \rdag$ and $\ms' = s'$:
        \(
            \mP(\ms' \mid \ms, \ma) = P(s' \mid s, a).
        \)
    \end{enumerate}
    All other transitions have probability $0$.

    \item Let $\ms \in \mS$, and $\ma \in \mA(\ms)$. If $\ms = (s,o,a,r)$ and $\ma = \rdag$, then $\mr(\ms, \ma) = g(s,a,\rdag)$. For all other meta-states, $\mr(\ms, \ma) = 0$.

    \item $\overline{\gamma} = \gamma^{1/4}$.

    \item $\mmu(s) = \mu(s)$ for $s \in \S$ and $\mmu(\ms) = 0$ otherwise.
\end{itemize}
\end{definition}

\paragraph{Reward Subtlety.} 
Note that the attacker only receives a reward at every fourth subtime.
This means the discount factor has to be ``slowed down'' so that the factor at every fourth time step matches that of each single time step of $M$. Specifically, choosing $\overline{\gamma} = \gamma^{1/4}$ ensures that $\overline{\gamma}^{4t} = \gamma^t$.

\begin{proposition}\label{prop: optimal-attack}
    The maximum expected reward the attacker can achieve from any attack on $\pi$ is $V^*_{\mM}$, the maximum expected total discounted reward for the meta-MDP $\mM$. Furthermore, any optimal deterministic, stationary policy $\nu^*$ for $\mM$ is an optimal attack policy.
\end{proposition}

\paragraph{Online Interaction.} Suppose the attacker has computed some attack policy $\nu$ from $\mM$. In order to use $\nu$ to interact with the victim and $M$, the attacker must know the meta-state at any given subtime. As long as the attacker can observe the interaction between the victim policy $\pi$ and $M$, it can effectively simulate the interaction with the meta-MDP $\mM$ online using a constant amount of memory. At time $t$, the attacker only needs to store $s_t$, $o_t$, $a_t$, and $r_t$ when they are revealed to the attacker. With this information, the attacker knows the meta-state for each subtime and so can apply $\nu$ to determine its next attack. 
Upon reach the next time $t+1$, the attacker can forget $s_t$, $o_t$, $a_t$, and $r_t$ and start from $s_{t+1}$. See \cref{alg: protocol}. 
\begin{figure}
    \centering
    \begin{algorithm}[H]
    \caption{Attacker Interaction Protocol}\label{alg: protocol}
    \textbf{Input:} $(\pi, \nu)$
    \begin{algorithmic}[1]
    \FOR{$t = 1 \ldots $}
        \STATE Attacker sees $s_t$, and computes a state attack $\sdag_t = \nu(s_t)$
        \STATE Attacker sees $o_t \sim \O(\sdag_t)$, and computes $\odag_t = \nu(s_t,o_t)$
        \STATE Attacker sees $a_t \sim \pi(\odag_t)$, and computes $\adag_t = \nu(s_t,o_t,a_t)$
        \STATE Attacker sees $r_t \sim \R(\sdag_t, \adag_t)$, and computes $\rdag_t = \nu(s_t,o_t,a_t,r_t)$
        \STATE Attacker receives reward $g(\sdag_t, \adag_t, \rdag_t)$, and forgets $(s_t, o_t, a_t, r_t)$
    \ENDFOR
    \end{algorithmic}
    \end{algorithm}
\end{figure}

\paragraph{Solving $\mM$.} If the attacker has full knowledge of $M$ and the victim's policy $\pi$, then the attacker has all the knowledge needed to construct the meta-MDP $\mM$. Once $\mM$ is constructed, the attacker can use any planning algorithm, such as policy iteration, to compute the optimal attack. Alternatively, if the attacker does not know $M$ and $\pi$, it can still simulate interacting with $\mM$ online as described before to perform learning. In particular, the attacker can replace the call to $\nu$ in \cref{alg: protocol} with any off-the-shelf learning algorithm. For the episodic setting, we view the attacker as attacking a new victim following the same policy $\pi$ in each episode.

Observe that $|\mS| \leq |\S||\O||\A||\R|$, $|\mA| \leq |\S| + |\O| + |\A| +|\R|$, and $\overline{\gamma} = \gamma^{1/4}$. Thus, whenever $M$'s rewards are finitely supported, $|\mM| = poly(|M|)$, where $|M|$ is the total size of $M$'s description. As such, any polytime planning algorithm or polynomial sample-complexity learning algorithm applied to $\mM$ yields an algorithm for computing optimal attacks that has polynomial complexity.

\begin{proposition}\label{prop: efficient-attack}
    When $M$'s rewards have finite support or no reward attacks are allowed, $|\mM| = poly(|M|)$. Thus, an optimal attack policy can be computed in polynomial time by planning in $\mM$, and learning an optimal attack policy can be performed with polynomial sample complexity by learning in $\mM$. 
\end{proposition}

\begin{remark}[Restricted Surfaces]
By restricting $\mA$ to singleton sets (e.g. set $\mA(s,o,a) = \{a\}$ to disallow action attacks), $\mM$ recovers optimal attacks for each individual surface as well as attacks for any subset of available attack surfaces. This captures all standard test-time attacks, generalizing the perceived-state attack MDP of \cite{Zhang2020}. We also note if the attacker does not perform reward attacks, $\mM$ can be modified to avoid $\R$ and so $M$ having finite supported rewards is unnecessary in the complexity results.
\end{remark}

One might ask whether the perceived-state attack MDP defined in \cite{Zhang2020} would work in the linear setting. We point out that the transition takes the following form, 
\begin{equation*}
\begin{aligned}
    \tilde P(s' \mid s, \sdag) &= \E_{a \sim \pi(\sdag)} P(s' \mid s, a) 
    \\ &= \int_{a} P(s' \mid s,a) \pi(a \mid \sdag) d a.
\end{aligned}
\end{equation*}
As $\pi$ and $P$ are multiplied together,  $\tilde P$ would be a quadratic transition. On the other hand,
our particular choice of subtimes induces linear structure in $\mM$.  Specifically, each transition of $\mP$ is defined by a single distribution involving $\pi$ or $M$. If both $\pi$ and $M$ have a linear structure, then so will $\mM$. Then, $\mM$ can be solved by standard linear RL algorithms. Thus, so long as $\pi$ is linear, the attacker can compute optimal attacks on linear environments.

\begin{theorem}\label{thm: linear}
    If $M$ is linear and $\pi$ is linear, then $\mM$ is linear. Furthermore, the dimension of $\mM$, $d(\mM)$, is at most $\max\{d(\pi), d(M)\} + 1$. Thus, if $\pi$ is linear, optimal attacks on linear environments can be computed or learned efficiently
\end{theorem}

\begin{remark}[Beyond Markovian Policies]\label{rem: beyond-markovian}
    Our construction can be easily modified to handle non-Markovian victim policies. If the victim uses some finite amount of past history $\tilde\hist$, we simply modify the meta-state space to remember the same amount of past history and adjust the construction appropriately. The size of $\mM$ is now a polynomial in both $|M|$ and the size of the policy when described explicitly as a mapping from histories to action distributions. We defer the details to the Appendix.
\end{remark}

\section{Optimal Defense}\label{sec: defense}
Now that we have seen how the attacker can best attack, it begs the question of how the victim should defend against attacks. Intuitively, the victim should choose a defense policy that is robust to attack. However, it does not suffice to just be robust against a particular attack. In fact, the attacker could lie about its attack algorithm to bait the victim into choosing a policy that actually benefits the attacker. Even if some attacker does use that particular attack algorithm, other attackers may employ different methods that lead the victim to poor value. As new attacks are formulated, the victim would have to constantly create more complex policies designed with all known attacks in mind. This would become a never-ending cat-and-mouse game during which the victim's policy will often be at risk of new attacks. Thus, for a policy to be satisfactorily robust, we require it to be robust against the worst possible attack. This way, no matter what future strategies an attacker may use, the victim is already prepared. 

We can formalize this intuition using the Stackelberg approach for Security Applications~\cite{StackSecurity}. For any $\pi$ and $\nu$, let $V_1^{\pi, \nu}$ and $V_2^{\pi,\nu}$ denote the victim's and attacker's expected reward respectively under the victim-attacker-$M$ interaction induced by $\pi$ and $\nu$. Note, both of these quantities can be computed efficiently using the previous section's techniques. Let $V_1$ and $V_2$ denote infinite matrices whose $(\pi,\nu)$ entry corresponds to $V_1^{\pi,\nu}$ and $V_2^{\pi,\nu}$ respectfully.
We define an infinite bimatrix game $G$ whose payoff matrices are $(V_1, V_2)$. For any fixed victim $\pi$, it is clear that a rational attacker would play some best-response policy, $\nu \in BR(\pi) := \max_{\nu \in N} V^{\pi, \nu}_{2}$.
Thus, an optimal defense policy is exactly an optimal Stackelberg strategy for player $1$ in $G$~\cite{Stackelberg}. 

\begin{definition}[Defense Problem]\label{def: defense}
The victim seeks a policy $\pi^*$ that maximizes its expected reward from the victim-attacker-$M$ interaction under the worst-case attack:
\begin{equation}\label{equ: defense}
    \pi^* \in \max_{\pi \in \Pi} \min_{\nu \in BR(\pi)} V^{\pi, \nu}_{1}.
\end{equation}
\end{definition}
Observe that this solution is truly robust: by definition, the attacker given $\pi$ would never want to deviate from $BR(\pi)$, and similarly, by definition the victim would never want to deviate from its defense policy when assuming the worst possible attack. Thus, we consider such attack and defense policies as truly \emph{optimal}. However, as the victim faces partial observability, an optimal defense for the victim is history-dependent in general. Consequently, the attacker's best response must also be history-dependent. Thus, $\Pi$ and $N$ consist of history-dependent policies in the definition above.

Although optimal Stackelberg strategies for Stochastic games are generally difficult to compute~\cite{StackMG}, we can exploit the special structure of the victim-attacker-$M$ interaction to develop useful algorithms. Recall that at subtime $t_2$ in \cref{alg: protocol}, the attacker changes the observation to $\odag$, and then the victim chooses an action $a = \pi(\odag)$. If we simply give the victim the autonomy to choose any action $a$ at this point rather than according to a fixed policy $\pi$, then this interaction evolves like a turn-based game. In fact, we show this game can be modeled as a partially observable turn-based stochastic game (POTBSG)~\cite{POTBSG}. POTBSGs exhibit much more structure than a general imperfect-information stochastic game, so enable more efficient solution methods. We see the construction is almost identical to \cref{def: meta-MDP}. 

\begin{definition}\label{def: potbsg}
The victim-attacker's POTBSG is $\overline{G} = (\mS_1 \cup \mS_2, \overline{\O}, \mA, \mP, \mr, \overline{\gamma}, \allowbreak \mmu)$ where, 

\begin{itemize}
    \item $\mS_1 := \S \times \O \times \set{\varnothing}$ and $\mS_2 := \S \cup (\S \times \O) \cup (\S \times \O \times \A) \cup (\S \times \O \times \A \times \R)$.

    \item $\overline{\O}(\ms) := o$ for $\ms = (s, o , \varnothing)$ and $\overline{\O}(\ms) := \ms$ otherwise. 
    
    \item $\mA(s) := \mathcal{B}(s)$, $\mA(s,o) := \mathcal{B}(s,o)$, $\mA(s,o,\varnothing) := \A$, $\mA(s,o,a) := \mathcal{B}(s,o,a)$, and $\mA(s,o,a,r) := \mathcal{B}(s,o,a,r)$.

    \item Let $\ms \in \mS$, $\ma \in \mA(\ms)$, and $\ms' \in \mS$.

    \begin{enumerate}
        \item If $\ms = s$, then $\ma = \sdag$ and $\ms' = (\sdag, o)$: \\ $\mP(\ms ' \mid \ms, \ma) := \O(o \mid \sdag).$
        \item If $\ms = (s,o)$, then $\ma = \odag$ and $\ms' = (s, \odag, \varnothing)$: \\ $\mP(\ms' \mid \ms, \ma) := \pi(a \mid \odag).$
        \item If $\ms = (s,o, \varnothing)$, then $\ma = a$ and $\ms' = (s, \odag, a)$: \\ $\mP(\ms' \mid \ms, \ma) := 1.$
        \item If $\ms = (s,o,a)$, then $\ma = \adag$ and $\ms' = (s, o , \adag, r)$: \\
        $\mP(\ms' \mid \ms, \ma) := R(r \mid s, \adag).$
        \item If $\ms = (s,o,a,r)$, then $\ma = \rdag$ and $\ms' = s'$: \\ $\mP(\ms' \mid \ms, \ma) := P(s' \mid s, a).$
    \end{enumerate}
    All other transitions have probability $0$.
    
    \item Let $\ms \in \mS$, and $\ma \in \mA(\ms)$. $\mr_1(\ms, \ma) := \rdag$ and $\mr_2(\ms, \ma) := g(s,a,\rdag)$ if $\ms = (s,o,a,r)$ and $\mr_1(\ms, \ma) := \mr_2(\ms, \ma) := 0$ otherwise.
    
    \item $\overline{\gamma} := \gamma^{1/5}$.

    \item $\mmu(s) := \mu(s)$ for $s \in \S$ and $\mmu(\ms) := 0$ otherwise.
\end{itemize}
\end{definition}

Note that $\mS_1$ is the set of states in which the victim takes an action, and $\mS_2$ is the set of states in which the attacker takes an action. The observation and action set $\overline{\O}$ and $\mA$ as functions of the states are combined for the two players, and this implies that the observations and actions for the victim are $\mA(\mS_1)$ and $\overline{\O}(\mS_1)$, and for the attacker are $\mA(\mS_2)$ and $\overline{\O}(\mS_2)$.
Observe that $V_{\mG, 1}^{\pi, \nu} = V_{1}^{\pi, \nu}$ and $V_{\mG, 2}^{\pi, \nu} = V_{\mM}^{\pi,\nu}$ and so $G$ is just the normal-form representation of the POTBSG $\mG$. 

\begin{proposition}\label{prop: optimal-defense}
    Any WSE for $\mG$ yields an optimal defense policy. 
\end{proposition}

In general, methods to compute WSE are unknown. However, we show many settings where a WSE for $\mG$ can be computed, even efficiently. First, suppose the attacker is completely adversarial so that $\mG$ becomes a zero-sum game. In this case, it is known that $WSE=SSE=NE$. Thus, it suffices to compute an NE for a zero-sum POTBSG.

\begin{proposition}\label{prop: zero-sum-defense}
    If the attacker is completely adversarial, an optimal defense policy can be computed as an NE of $\mG$ using any planning or distributed learning algorithms for zero-sum POTBSGs. 
\end{proposition}

Note, it is important that the victim uses a distributed learning algorithm since it would not be able to see the attacker's manipulations, only the effects of the manipulations, nor be able to collaborate with the attacker. From \cref{prop: optimal-defense}, we see that the victim can compute an optimal defense policy to an adversarial attacker by computing any CCE to $\mG$. However, even computing an approximately optimal Markovian policy against a fixed attack is equivalent to solving a POMDP, which is NP-hard~\cite{POMDP-hardness}. Thus, computing near-optimal defenses is intractable in the worst case.

\begin{proposition}\label{prop: defense-hardness}
    For any $\epsilon > 0$ an $\epsilon$-approximate optimal defense policy is NP-hard to compute even when restricting $\Pi$ and $N$ to be the class of Markovian policies. 
\end{proposition}

\paragraph{Efficient Methods.} The main bottleneck to computing defenses efficiently in fully-observable systems is the presence of perceived-state attacks. Absent these attacks, the POTBSG specializes to a traditional TBSG, which is a special case of a stochastic game.
\begin{observation}\label{obs: tbsg}
    When $M$ is fully observable and the attacker cannot perform perceived-state attacks, $\mG$ simplifies to a TBSG.
\end{observation}
In the adversarial case, we see that $\mG$ is simply a zero-sum TBSG. In zero-sum TBSGs, even stationary NE can be computed or learned efficiently~\cite{TBSGOptimalLearn} unlike the case with CCE for MGs~\cite{daskalakis2022complexity} and the solutions are exact. 

\begin{proposition}\label{prop: efficient-defense}
    If $M$ is fully-observable, no perceived-state attacks are allowed, and $M$'s rewards have finite support (or no reward attacks are allowed), and the attacker is adversarial, then an \emph{optimal stationary} defense policy can be computed in polynomial time and learned with polynomial sample complexity.
\end{proposition}

Although it is unclear whether Markovian policies guarantee the victim as much value as history-dependent ones, Markovian policies are commonplace since they are easier to store and deploy in practice. In fact, for the finite-horizon planning setting, the attacker need not be restricted. We give polynomial time planning algorithms to compute an optimal defense so long as perceived-state attacks are banned. To our knowledge, this is the first non-trivial setting for which WSE can be computed efficiently and the first non-trivial setting for which SSE can be computed beyond single-period games.

\begin{theorem}\label{thm: defense-planning}
    If $M$ is fully-observable and has a finite horizon, no perceived-state attacks are allowed, and $M$'s rewards have finite support (or no reward attacks are allowed), then an optimal defense policy can be computed in polynomial time.
\end{theorem}

The intuition is the victim can simulate the attacker's best-response function using backward induction. Once it knows the best response for a particular stage game, it can then brute-force find the best action to take at that stage. The key insight is that the attacker's best response is always deterministic since it gets to see the victim's realized actions. Thus, the victim also has no benefit from randomization. As such, the victim can brute-force compute its optimal deterministic action to take during a single stage and then propagate that solution backward to be used in previous times.

To illustrate this, we derive a backward induction algorithm for efficient defense against action attacks and present the full defense algorithm in the Appendix. Suppose the victim has already committed to $\{\pi^*_t\}_{t = h+1}^H$, where $H$ is the finite time-horizon. Clearly, for any choice of victim's action $a$, the attacker's best response to $a$ and the future partial policy is:
\begin{align*}
    BR_{h}(s,a) = &\argmax_{\adag \in \mA(s,a)} g_h(s,a,r_h(s,a)) \\
    & \hspace{1 em}+ \E_{s'\sim P_h(s, \adag)} V^*_{h+1,2}(s', \pi^*_{h+1}(s')),
\end{align*}
where $V^*_{h,2}(s,a)$ is the maximum value achieved. Then, the victim can compute its best action for the stage game $(h,s)$ as a maximizer of,
\begin{align*}
    V^*_{h,1}(s) = &\max_{a \in \A} \min_{\adag \in BR_h(s,a)} r_h(s, \adag) \\
    & \hspace{1 em} + \E_{s' \sim P_h(s, \adag)} V^*_{h+1,1}(s').
\end{align*}
 
The construction for defending against all non-perceived state surfaces is a bit more complicated but retains this same structure.

\begin{remark}[Multi-Agent Extension]\label{rem: multi-agent}
    We note that all of our results remain the same when multiple victims are present. This can be done without changing any of the previous notations by interpreting $\A = A_1 \times \ldots \times A_n$ as the joint action space and $\pi$ as a joint policy. From the attacker's perspective, attacking many victims just looks like attacking a single victim with a large action space. A WSE in $\mG$ still breaks up into an independent joint policy for the victims and the attacker, but the joint policy may require the victims to correlate with each other.
\end{remark}

\section{Experiments}\label{sec: experiments}

We illustrate our frameworks with a classical grid-world shortest path problem with obstacles. Here, each state is a cell in a $n \times n$ grid. Some grid cells are filled with lava and so dangerous to the victim. From any cell, the victim can move left (L), right (R), up (U), or down (D) so long as it remains on the grid. In addition, the victim can stay (S) in its current cell. The agent wishes to get from the top-left cell $(0,0)$ to the bottom-right, ``goal'', cell $(n-1,n-1)$ as quickly as possible while avoiding lava. To capture this goal, we assume the victim receives a reward of $1$ for entering the goal cell and continues to receive a reward of $1$ for each time it remains there to incentivize the victim to reach the goal quickly. We also assume the victim receives a penalty reward of $-H$ whenever it enters a lava cell, where $H$ is the finite horizon.

Here, we test our methods on a $10 \times 10$ grid world with $H = 20$ so that the victim has enough time to reach the goal and stay there. We computed an optimal policy $\pi^*$ for the grid, which achieves the victim a value of $3$. In \cref{fig: clean} we visualize $\pi^*$ through the path the victim follows when using $\pi^*$. The black cells represent a cell the victim entered during its interaction. The orange cells represent lava.

\begin{figure}
    \centering
    \includegraphics[scale=.3]{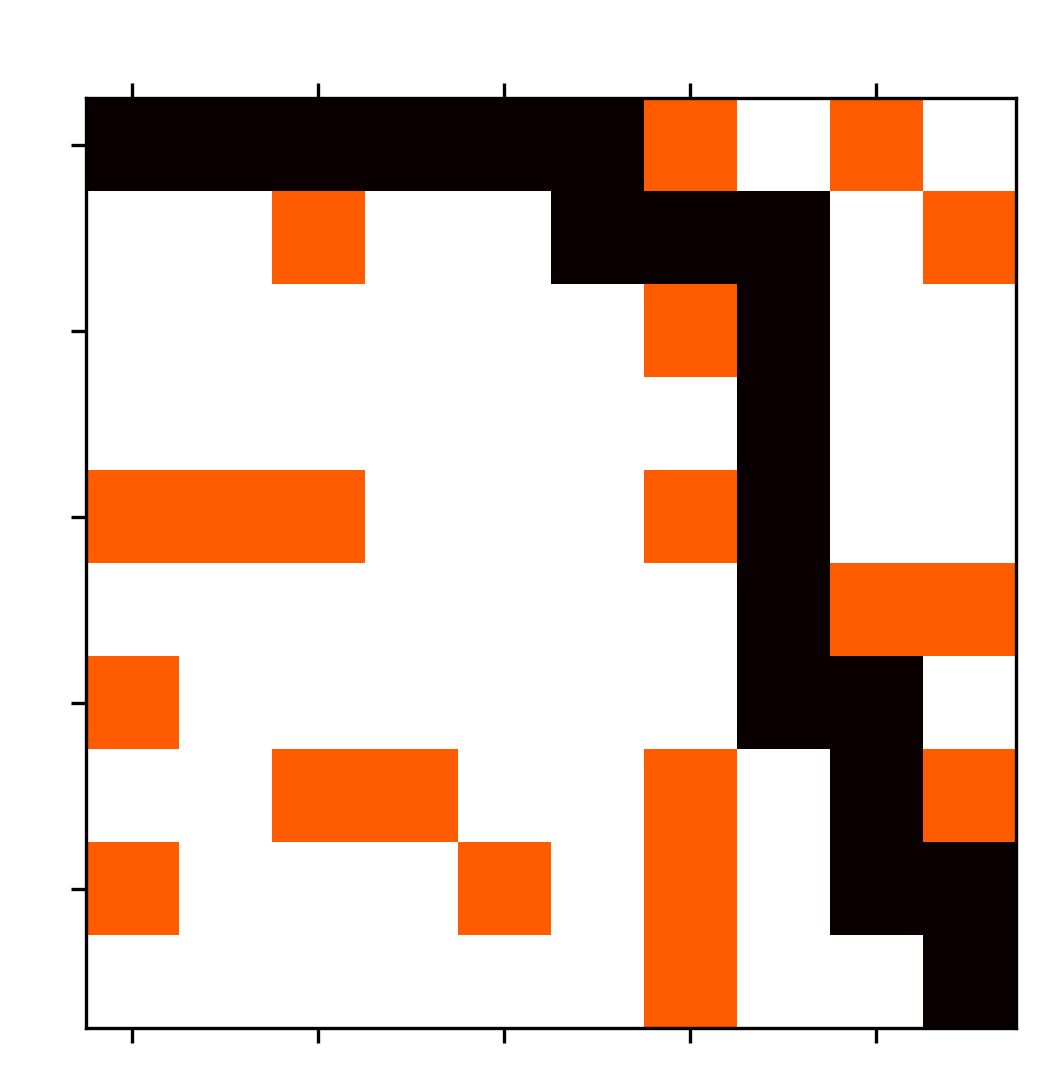}
    \caption{Optimal Policy Path.}
    \label{fig: clean}
\end{figure}

\subsection{Grid Attacks}\label{subsec: maze}

The attacker can utilize its surfaces to disrupt the victim's path. 
For simplicity, assume that the attacker is purely adversarial and so it seeks to prevent the victim from reaching the goal and even trick it into lava cells if possible. Suppose that most of the grid is under security and so attacks cannot be safely made. The attacker is restricted to only attacking edges of the grid, which are not monitored. Here, the regions include the top-right subgrid and the bottom-left subgrid shaded in yellow. However, in those regions, it may use any attack it likes from its given surface.

\begin{figure}
    \centering
    \includegraphics[scale=.3]{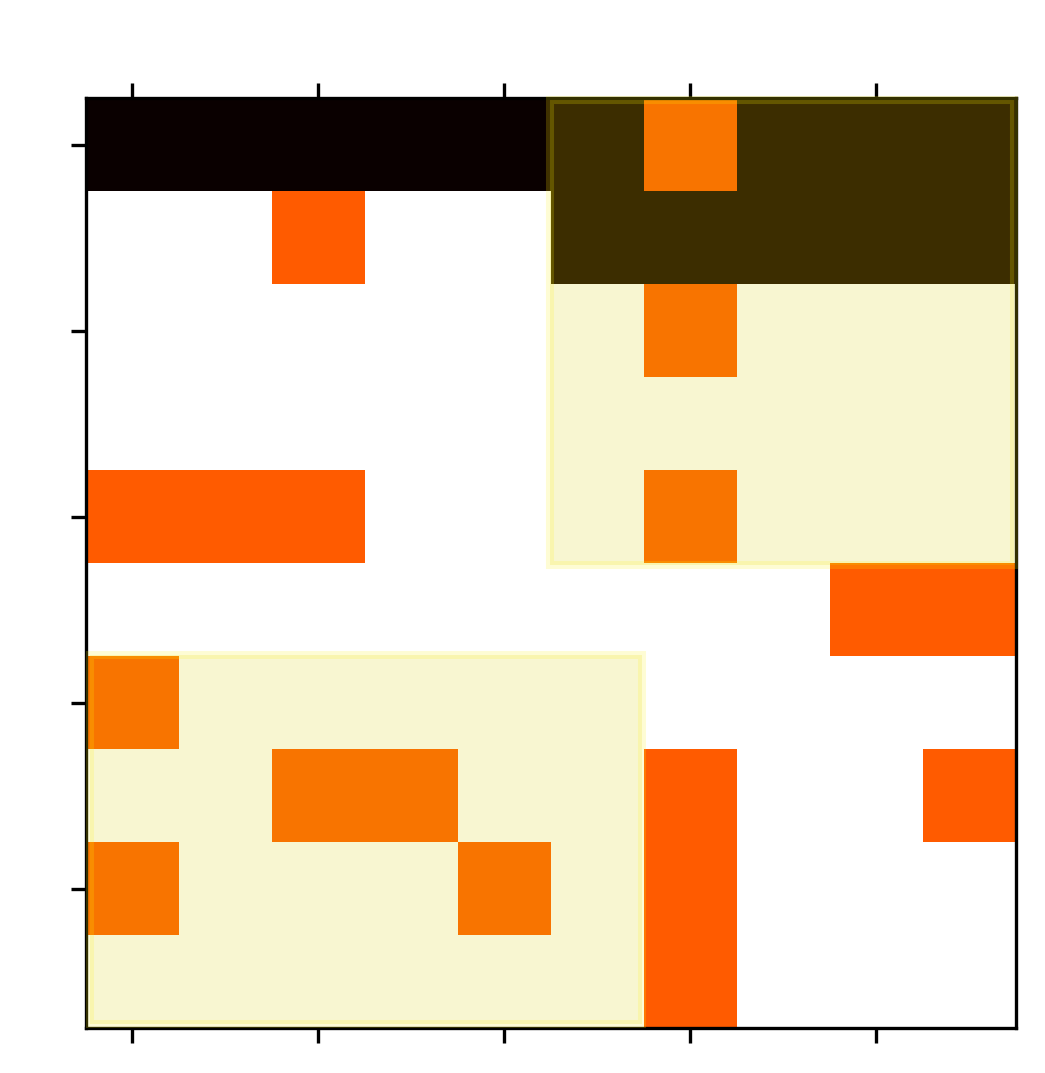}
    \includegraphics[scale=.3]{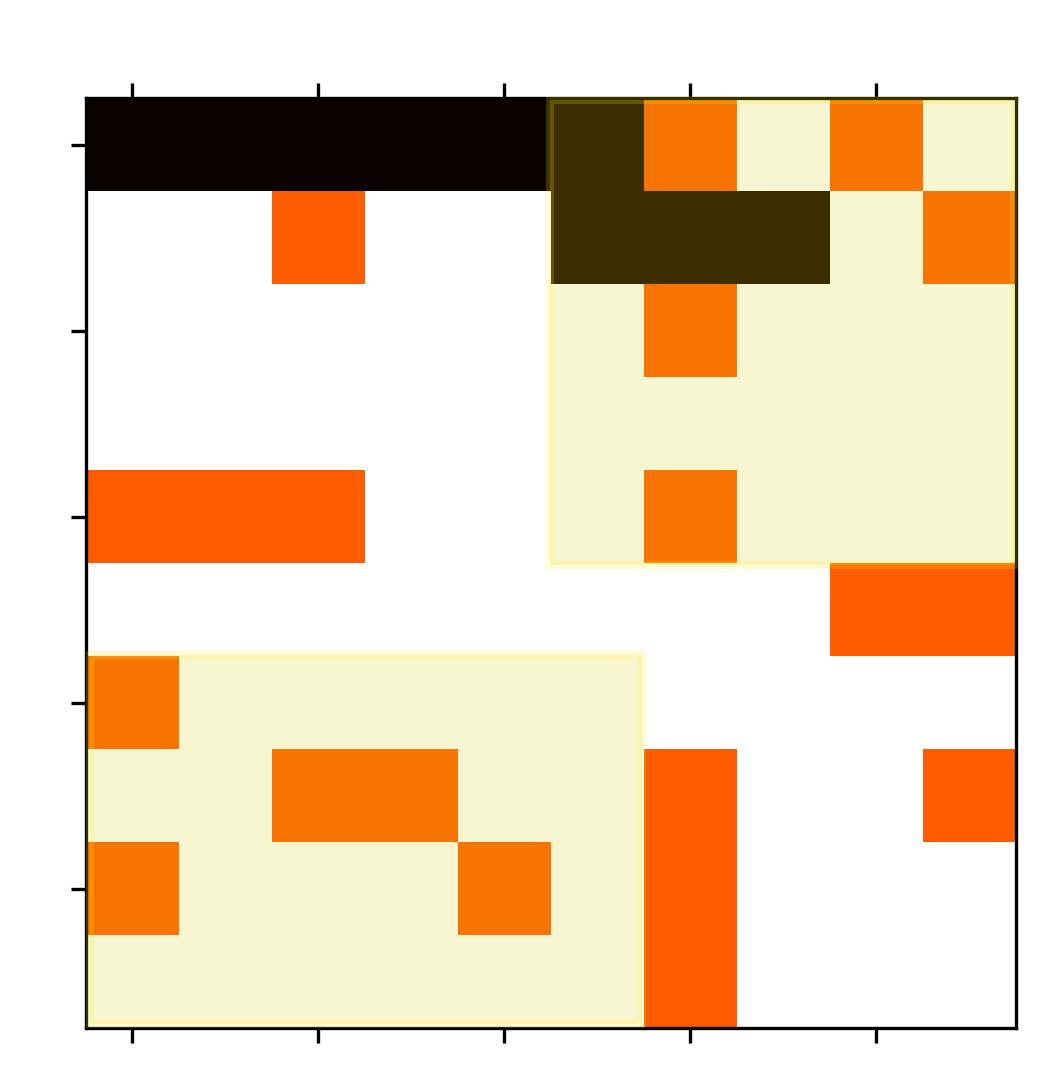}
    \includegraphics[scale=.3]{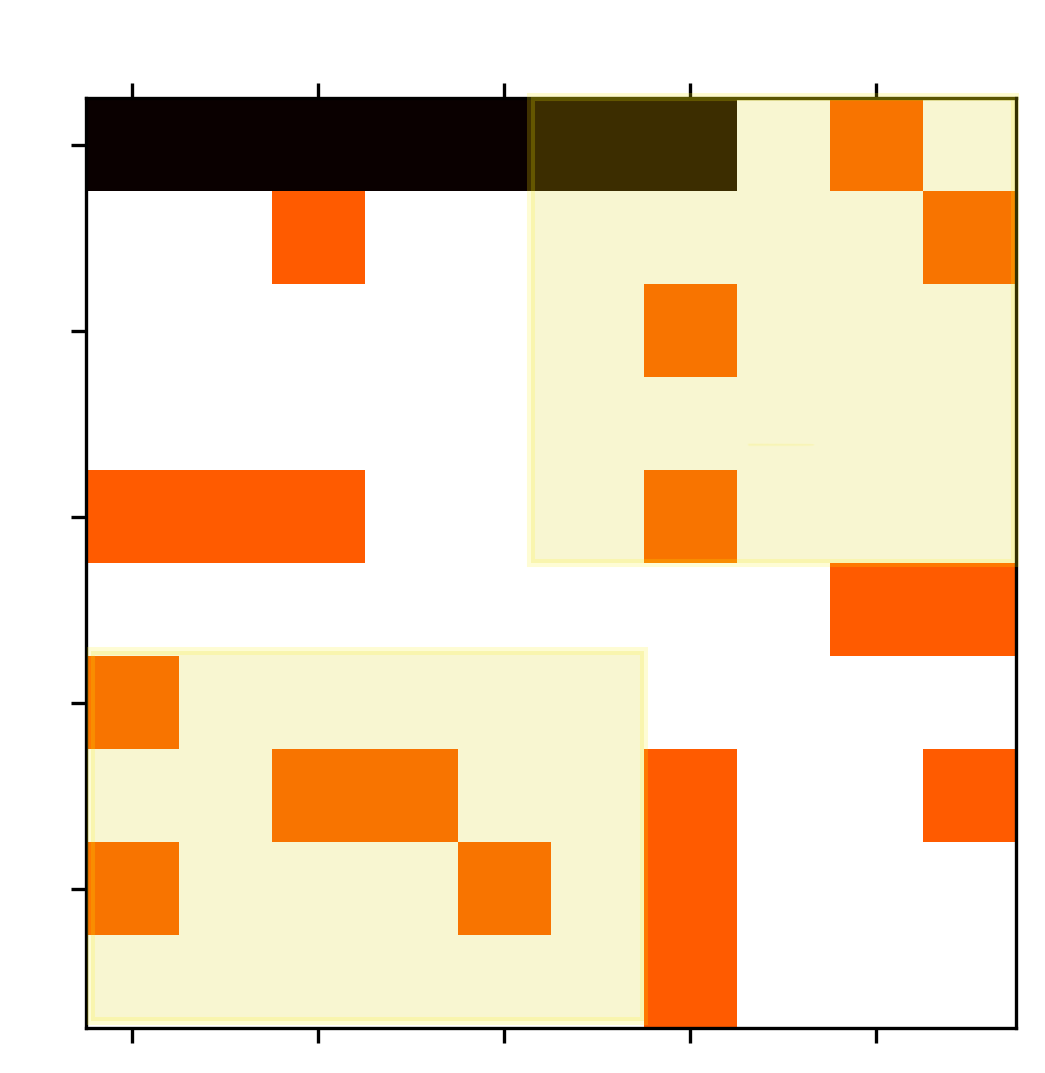}
    \caption{Attacked Paths.}
    \label{fig: attacks}
\end{figure}

In \cref{fig: attacks}, we see from left to right the path under an optimal perceived-state attack, true-state attack, and action attack. The agent receives $-100$, $0$, and $-160$ value from each attack respectively. In all cases, the victim no longer reaches the goal after getting attacked in the top-right subgrid. We see the perceived-state attack functions by tricking the agent into entering lava; whereas the action attack simply forces the victim into lava. On the other hand, the state attacks can transport the victim into lava, but they immediately leave and so suffers less damage than in the other attacks despite seeming to be the most powerful.

\subsection{Grid Defense}

We see that if the victim simply follows $\pi^*$, the effects of attacks can be catastrophic. 
The victim knows the upper-right and bottom-left subgrids are not monitored and so can assume attacks are conducted there. Using this information, the defense algorithm yields a policy $\hat \pi$ that completely avoids the unsafe region. The victim still achieves the optimal value of $3$ even under the strongest-possible attack. The new path under attack is illustrated in \cref{fig: defense}. We see the robust path simply squeezes between the two unsafe regions.

\begin{figure}
    \centering
    \includegraphics[scale=.3]{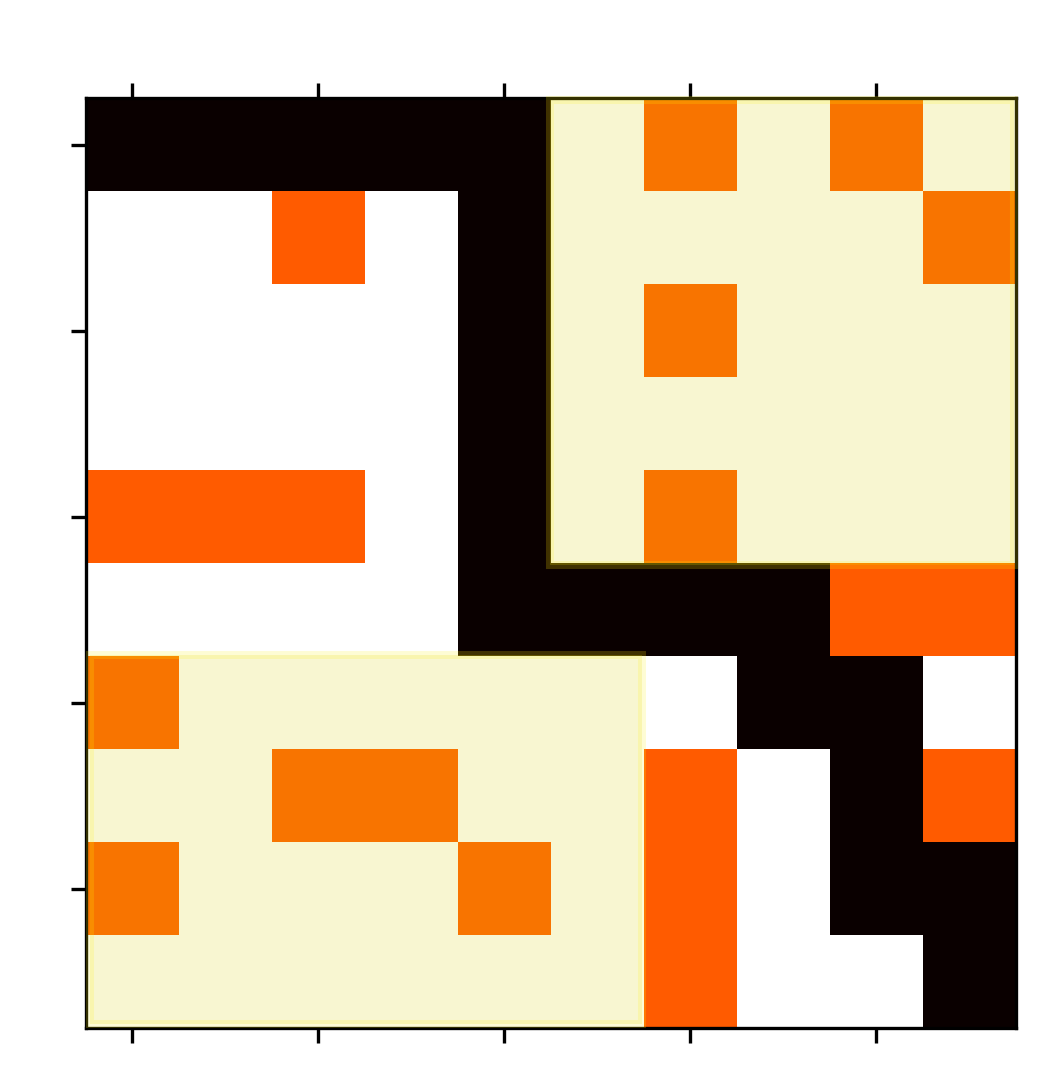}
    \caption{Defense Policy Path}
    \label{fig: defense}
\end{figure}

\section{Conclusions}\label{sec: conclusions}

In this paper, we rigorously studied the attack and defense problems of reinforcement learning. We showed that for any attack's surface, a malicious attacker can optimally and efficiently maximize its own rewards by solving a higher lever meta-MDP. Even against linear environments, an attacker can still efficiently compute optimal attacks. Thus, we call for an agent to play a robust policy to be safe against such attacks. To this end, we formally defined the defense problem to be a weak-Stackelberg equilibrium of the natural partially-observable turn-based stochastic game that is induced by the victim-attacker-environment interaction. In the zero-sum setting, we showed the defense problem boils down to finding a Nash equilibrium in a zero-sum POTBSG so standard planning and learning methods can find optimal defense policies. When perceived-state attacks are not allowed, the victim can also compute an optimal defense policy in polynomial time using a robust backward induction algorithm. Although we present an optimal defense, this defense may not be useful if the attacker is too powerful. It is critical for the victim to improve its detection abilities to restrict the attacker's feasible actions.

\section*{Acknowledgments}
This project is supported in part by NSF grants 1545481, 1704117, 1836978, 2023239, 2041428, 2202457, ARO MURI W911NF2110317, and AF CoE FA9550-18-1-0166. Xie is partially supported by NSF grant 1955997. We also thank Yudong Chen for his useful comments and discussions.

\bibliography{aaai24}

\onecolumn

\appendix

\section{Proofs for Optimal Attacks}\label{app: proofs}

The proof of the \cref{prop: optimal-attack} is immediate from the definition of each attack surface and the well-known fact that any MDP admits an optimal deterministic, Markovian policy. \cref{prop: efficient-attack} follows from the complexity results given in ~\cite{MDP-book} and ~\cite{PAC}.

\paragraph{Proof of \cref{thm: linear}.}

\begin{proof}
    The key idea is each meta-state transitions according to either, $O$, $\pi$, $R$, or $P$, and since each of these quantities is linear, the meta-transitions are linear. Also, as the rewards are a deterministic projection, they are also linear.
    If $M$ and $\pi$ are linear then,
    \[O(o \mid s) = \inner{\phi(s),\gamma(o)}, \quad \pi(a \mid o) = \inner{\psi(o), \delta(a)}, \quad R(r \mid s,a) = \inner{\phi(s,a),\theta(r)}, \quad P(s'\mid s,a) =  \inner{\phi(s,a),\mu(s')}.\]
    We first design a feature vector $\bar \phi(\ms, \ma)$ and vector $\mmu(\ms)$ that captures the transitions. Let $\ms \in \mS$, $\ma \in \mA(\ms)$, and $\ms' \in \mS$. From the definition of $\mM$,

    \begin{enumerate}
        \item If $\ms = s$, then $\ma = \sdag$ and $\ms' = (\sdag, o)$: \begin{equation*}
            \mP(\ms ' \mid \ms, \ma) := \O(o \mid \sdag) = \inner{\phi(\sdag),\gamma(o)}.
        \end{equation*}
        Define $\bar \phi(s, \sdag) = \phi(\sdag)$ and $\mmu(\sdag,o) = \gamma(o)$. Then,
        \[\inner{\phi(\sdag),\gamma(o)} = \inner{\bar \phi(\ms,\ma), \mmu(\ms')} .\]
        \item If $\ms = (s,o)$, then $\ma = \odag$ and $\ms' = (s, \odag, a)$: 
        \begin{equation*}
            \mP(\ms' \mid \ms, \ma) := \pi(a \mid \odag) = \inner{\psi(\odag), \delta(a)}.
        \end{equation*}
        Define $\bar \phi((s,o), \odag) = \psi(\odag)$ and $\mmu(s,\odag,a) = \delta(a)$. Then,
        \[\inner{\psi(\odag),\delta(a)} = \inner{\bar \phi(\ms, \ma), \mmu(\ms')} .\]
        \item If $\ms = (s,o,a)$, then $\ma = \adag$ and $\ms' = (s, o , \adag, r)$: 
        \begin{equation*}
            \mP(\ms' \mid \ms, \ma) := R(r \mid s, \adag) = \inner{\phi(s,\adag),\theta(r)}.
        \end{equation*}
        Define $\bar \phi((s,o,a), \adag) = \phi(s,\adag)$ and $\mmu(s,o,a,r) = \theta(r)$. Then,
        \[\inner{\phi(s,\adag),\theta(r)} = \inner{\bar \phi(\ms,\ma), \mmu(\ms')} .\]
        \item If $\ms = (s,o,a,r)$, then $\ma = \rdag$ and $\ms' = s'$:
        \begin{equation*}
            \mP(\ms' \mid \ms, \ma) := P(s' \mid s, a) = \inner{\phi(s,a),\mu(s')}.
        \end{equation*}
        Define $\bar \phi((s,o,a,r), \rdag) = \phi(s,a)$ and $\mmu(s') = \mu(s')$. Then,
        \[\inner{\phi(s,a),\mu(s')} = \inner{\bar \phi(\ms,\ma), \mmu(\ms')} .\]
    \end{enumerate}
    We lift each vector to dimension $d = \max(d(\pi), d(M))$ so each vector is in the same dimension.
    
    Now, to capture rewards, we will need to lift the vectors to one dimension higher. We add one entry to each vector. The entry is $g(s,a,\rdag)$ for meta-states of form $(s,a,r)$ and meta-action $\rdag$, making the new vector $\phi'(\ms,\ma) = \begin{bmatrix}
           g\left(s,a,\rdag\right) \\ \bar \phi(\ms,\ma)  
         \end{bmatrix}$. All other meta-states and meta-actions have a $0$ in the new entry, or $\phi'(\ms,\ma) = \begin{bmatrix}
           0 \\ \bar \phi(\ms,\ma)  
         \end{bmatrix}$. We define $\bar \theta = e_1$, the basis vector with a $1$ in the entry that we just added to the features. From the definition of $\mM$,
        \[\mr(s, \sdag) = \mr((s,a), \adag)  = \mr((s,a), \sdag) = 0 = \inner{\bar \phi'(\ms,\ma), e_1}.\]
        \[\mr((s,a,r), \rdag) = g(s,a,\rdag) = \inner{\bar \phi'(\ms,\ma), e_1}.\]
     Thus, $\mM$ is linear with dimension at most $1 + \max(d(\pi), d(M))$.
\end{proof}
We note that for the finite-horizon case, the proof remains the same except for one aspect. If the attacker's reward $g$ is time-dependent, then the feature vectors involving $\rdag$ must save $g_h(s,a,\rdag)$ for all $h$. Thus the dimension would be $H+1 + \max(d(\pi),d(M))$ if we modify the proof in the obvious way.

\paragraph{More Details of \cref{rem: beyond-markovian}.}
When the victim's policy is not Markovian, i.e. uses some amount of history, the construction of $\mM$ needs to be slightly modified. Suppose $\mathcal{M}$ denotes the victim's finite memory. For any $m \in \mathcal{M}$, we let $m(\cdot)$ denote the victim's updated memory upon receiving new information. Since the victim has a finite amount of memory, we assume that when the memory is updated, it removes the oldest saved information to make space for the new information. We also assume here, the attacker is attacking a Markov game $G$ so that $\pi$ is now a joint policy as described in \cref{rem: multi-agent}. We define the attacker's meta-MDP as $\mM = (\mS, \mA, \mP, \mr, \gamma, \mmu)$, where, 

\begin{itemize}
    \item $\mS = (\mathcal{M} \times \S) \cup (\mathcal{M} \times \S \times O) \cup (\mathcal{M} \times \S \times O \times \A) \cup (\mathcal{M} \times \S \times O \times \A \times \R)$.
    
    \item $\mA \subseteq \S \cup \A \cup \R$ consists of all the attacker's potential manipulations of the interaction. The meta-action space is meta-state dependent: $\mA(m,s) \subseteq \S$, $\mA(m,s,o) \subseteq \O$, $\mA(m,s,o,a) \subseteq \S \cup \A$, and $\mA(m,s,o, a, r) \subseteq \R$.
    
    \item Suppose that $m \in \mathcal{M}$, $s \in \S$, $o \in O$, $a \in \A$, and $r \in \R$. Then,

    \begin{itemize}
        \item If $\ma = \sdag$ is a true-state attack at the first subtime, then $G$'s state becomes $\sdag$ and generates an observation according to $P(\sdag)$.
        \[\mP((m, \sdag, o) \mid (m,s), \sdag) := P(o \mid \sdag).\]
        \item If $\ma = o^{\dagger} \in O$ is an observation attack, then the agents choose a joint action $a \sim \pi(m,\odag)$. Also, the agent sees the attacked observation and so updates its memory to $m(\odag)$.
        \begin{align*}
            \mP((m(\odag),s, \odag,a) &\mid (m,s,o), o^{\dagger}) := \pi(a \mid m, \odag).
        \end{align*}
        \item If $\ma = \adag \in \A$ is an action attack,  $G$ receives action $\adag$. Thus, $G$ generates reward according to $R(s,\adag)$. The agents also update their memory to $m(a)$.  
        \begin{align*}
            \mP((m(a),s, o, \adag, r)  &\mid (m,s,o,a), \adag) := R(r \mid s,\adag).
        \end{align*}
        \item If $\ma = \sdag \in \S$ is a true-state attack at the third subtime, $G$'s state becomes $\sdag$. Thus, $G$ generates reward according to $R(\sdag,a)$. The agents also update their memory to $m(a)$.   
        \begin{align*}
            \mP((m(a),\sdag, o, a, r)  &\mid (m,s,o,a), \sdag) := R(r \mid \sdag, a).
        \end{align*}
        \item If $\ma = \rdag \in \R$ is a reward attack, $G$'s transitions are not effected. Thus, $G$  transitions normally according to $P(s,a)$. Also, the agent sees the attacked reward and so updates its memory to $m(\rdag)$.
        \[\mP((m(\rdag),s') \mid (m,s,o,a,r), \rdag) := P(s' \mid s,a).\]
    \end{itemize}
    All other transitions have probability $0$.
    
    \item Using the same definitions as above,

    \begin{itemize}
        \item There is no immediate reward to the attacker for non-reward attacks since the agents have not received a reward yet.
        \begin{align*}
            &\mr((m,s), \sdag) = \mr((m,s,o),o^{\dagger}) = \mr((m,s,o,a), \adag)  = \mr((m,s,o,a), \sdag) = 0.
        \end{align*}
        
        \item If $\ma = \rdag \in \R$, the agents receive reward $\rdag$ and so the attacker receives reward according to its reward function $g(s,a,\rdag)$.
        \[\mr(\ms, \ma) = \mr((m,s,o,a,r), \rdag) = g(s,a,\rdag).\]
    \end{itemize}
    
    \item $\mmu(s) = \mu(\varnothing, s)$ for each $s \in \S$ where $\varnothing$ denotes the empty memory, and $\mmu(\ms) = 0$ otherwise.
\end{itemize}

\paragraph{Deterministic Markovian policy and deterministic reward.} If the agent's policy is deterministic and Markovian, which is the standard for optimal policies, and $M$ is fully-observable with deterministic rewards, the attacker's meta-MDP can be drastically simplified. We define the attacker's meta-MDP as $\mM = (\mS, \mA, \{\mP_h\}, \{\mr_h\}, H, \mmu)$, where,

\begin{itemize}
    \item $\mS = \S$.
    
    \item $\mA = (\S \times \{p\}) \cup \A \cup (\S \times \{t\}) \cup \R$, where $(s, p)$ is a perceived state attack and $(s, t)$ is a true state attack.
    
    \item $\mP_h(\ms' \mid \ms, \ma)$'s definition varies depending on the choice of $\ma$. Suppose that $s$ is the current state at time $h$. Then,

    \begin{itemize}
        \item If $\ma = (\sdag,p) \in \S$ is a perceived-state attack,
        \[\mP_h(s' \mid s, (\sdag,p)) = P_h(s' \mid s,\pi_h(\sdag)).\]
        \item If $\ma = \adag \in \A$, 
        \[\mP_h(s' \mid s, \adag) = P_h(s' \mid s,\adag).\]
        \item If $\ma = (\sdag,t) \in \S$ is a true-state attack, 
        \[\mP_h(s' \mid s, (\sdag,t)) = P_h(s' \mid \sdag,\pi_h(\sdag)).\]
        \item If $\ma = \rdag \in \R$, 
        \[\mP_h(s' \mid s, \rdag) = P_h(s' \mid s,\pi_h(s)).\]
    \end{itemize}
    and all other transitions have probability $0$ at time $h$.
    
    \item $\mr_h(\ms, \ma)$'s definition similarly depends on $\ma$ as well as the form of $\ms$. Using the same definitions as above, we have 

    \begin{itemize}
        \item If $\ma = (\sdag,p) \in \S$ is a perceived-state attack,
        \[\mr_h(s, (\sdag,p)) = g(s,r_h(s,\pi_h(\sdag))).\]
        \item If $\ma = \adag \in \A$, 
        \[\mr_h(s, \adag) = g(s,r_h(s,\adag)).\]
        \item If $\ma = (\sdag,t) \in \S$ is a true-state attack, 
        \[\mr_h(s, (\sdag,t)) = g(s,r_h(\sdag,\pi_h(\sdag))).\]
        \item If $\ma = \rdag \in \R$, 
        \[\mr_h(s, \rdag) = g(s,\rdag).\]
    \end{itemize}
    
    \item $\mmu(s) = \mu(s)$ for each $s \in \S$ and $\mmu(\ms) = 0$ otherwise.
\end{itemize}

\section{Proofs for Optimal Defense}

The proof of \cref{prop: optimal-defense} and \cref{prop: zero-sum-defense} is immediate from the definitions of the defense problem, the victim-attacker-environment interaction, and the definition of turn-based stochastic games~\cite{TBSGtheory}.

\paragraph{Proof of \cref{prop: defense-hardness}.}

\begin{proof}
    We show that the defense problem with observation or perceived-state attacks captures solutions to POMDPs. As such, all hardness results for POMDPs apply to Defense. Namely, computing $\epsilon$-optimal deterministic Markovian policies is NP-hard \cite{POMDP-hardness}. Also, for the discounted infinite horizon case, computing an optimal stationary stochastic memory-less policy is NP-hard \cite{Stochastic-Blind}.  
    
    We note any POMDP can be formulated as an equivalent POMDP with a deterministic observation function that is only polynomially sized larger. Thus, we can focus on POMDPs with a deterministic observation function $o : \S \to \O$. We can also assume that $\O \subseteq \S$. Given such a POMDP, define $M$ to be the same as the POMDP ignoring the observation part, and define the set constraints for a perceived-state attack to be the singleton $B(s) = \{o(s)\}$. This ensures the only feasible attack $\nu$ is exactly the observation function $o$. Thus, solving the defense problem for a maximum policy $\pi$ over some class of policies $\Pi$ is exactly equivalent to solving the POMDP over that class of policies $\Pi$. Hardness then follows.

\end{proof}

The proof of \cref{prop: efficient-defense} follows from \cref{obs: tbsg} and the stated complexity results.

\paragraph{Proof of \cref{thm: defense-planning}.} We present a formal backward induction algorithm for computing defenses. We define $V^*_{h,1}(s) := \max_{\pi \in \Pi_h} \min_{\nu \in BR_h(\pi)} V^{\pi,\nu}_{h,1}(s)$, where $\Pi_h$ is the set of Markovian partial policies from time $h$ forward, $BR_h(\pi) := \argmax_{\nu \in N_h} V^{\pi,\nu}_{h,2}(s)$, and $V^{\pi,\nu}_{h,i}(s)$ is the value of the stage game $(h,s)$ for player $i$ (i.e. their expected value from time $h$ onward starting from state $s$). It is clear that $V^*_{h,1}(s)$ admits optimal substructure: if $\pi^*_h(s)$ is a maximizer of $V^*_{h,1}(s)$, then it must be the case that $\pi^*_{h+1}(s')$ is a maximizer of $V^*_{h+1,1}(s')$ for any $s'$ reachable with non-zero probability under $\pi^*_h(s)$. Otherwise, more value could be achieved by improving the value in the future. This is the standard argument for why rollback successfully computes a subgame perfect equilibrium in sequential games. Here it is crucial that the defense problem is captured by a turn-based game, which gives us this sequential structure. Given this observation, it is clear that $V^*_{h,i}(s)$ can be computed recursively as follows. Suppose the victim already computed its optimal $\pi^*_{h+1}$. Then, $\pi^*_h(s)$ can be computed as a maximizer of $V^*_{h,1}(s)$, where:

\begin{equation}
    V^*_{h,1}(s) := \max_{a \in \A} \min_{\adag \in BR_h(s,a)} \E_{r \sim R_h(s,\adag)}\left[\min_{\rdag \in BR_h(s,\adag,r)} \rdag + \E_{s' \sim P_h(s, \adag)}\brac{\min_{\sdag \in BR_{h+1}(s')} V_{h+1,1}^*(\sdag)}\right].
\end{equation} 

Each best-response function is defined as follows. In each case, the attacker must wait for further information and then adapt to each realization. Hence the separate best-response functions.

\begin{equation}
    BR_{h}(s,a) := \argmax_{\adag \in \mA_h(s,a)} \E_{r \sim R_h(s,\adag)} V^*_{h,2}(s,a,r).
\end{equation}

\begin{align}
    BR_{h}(s,a,r) := \argmax_{\rdag \in \mA_h(s,a,r)} g_h(s,a,\rdag) + \E_{s' \in P_h(s,a)} V^*_{h+1}(s').
\end{align}

\begin{equation}
    BR_{h+1}(s) := \max_{\sdag \in \mA_{h+1}(s)} V^*_{h+1,2}(\sdag, \pi_{h+1}(\sdag)))
\end{equation}

In all cases, $V^*_{h+1,2}(\cdot)$ is defined to be the maximizing value for each corresponding $BR$ set.

In the zero-sum case, we get the even simpler expression,

\begin{equation}
    V^*_{h,1}(s) = \max_{a \in \A} \min_{\adag \in \mA_h(s,\adag)} \E_{r \sim R_h(s,\adag)}\left[\min_{\rdag \in \mA(s,\adag,r)} \rdag + \E_{s' \sim P_h(s, \adag)}\brac{\min_{\sdag \in \mA(s')} V^*_{h+1,1}(\sdag)}\right].
\end{equation}

Since turn-based games admit deterministic solutions, it is clear that the victim can compute its optimal defense in polynomial time by simply brute-force computing each $\max$ and $\min$ in the above expressions, which are finite optimizations in the tabular case.

\section{Code Details}

We conducted our experiments using standard python3 libraries. We provide our code in a jupyter notebook with the example grid being hard-coded in to ensure the experiments can be easily reproduced\footnote{Code can be found at \url{https://github.com/jermcmahan/Attack-Defense}.}.

\end{document}